\documentclass[12pt,natbib]{article}%
\usepackage[english]{babel}%
\usepackage[utf8]{inputenc}

\usepackage[T1]{fontenc}%
\usepackage{amsmath,amssymb,amsfonts,amsthm}%
\usepackage{graphicx,psfrag,epsf}%
\usepackage{dsfont}%
\usepackage{enumerate}%
\usepackage{natbib}%

\usepackage{longtable}
\usepackage{comment}
\usepackage[export]{adjustbox}
\usepackage{array}

\usepackage{mystyle}%
\usepackage{authblk}
\usepackage{hyperref}
\usepackage{booktabs}
\usepackage{lscape}
\usepackage{subcaption}

\title{MAC, a novel stochastic optimization method\thanks{This project is supported by Hungarian National Research, Development, and Innovation Office via grants FK134332 and K132109.}}
\author[1]{Attila L\'{a}szl\'{o} Nagy}
\affil[1]{Department of Stochastics, Institute of Mathematics,
Budapest University of Technology and Economics, 
M\H{u}egyetem rkp. 3., H-1111 Budapest, Hungary, (nagyal@math.bme.hu)}
\author[2]{Goitom Simret Kidane}
\author[2]{Tam\'{a}s Tur\'{a}nyi}
\affil[2]{Institute of Chemistry, ELTE E\"{o}tv\"{o}s Lor\'{a}nd University,
Budapest, Hungary}
\author[2,3]{J\'{a}nos T\'{o}th}
\affil[3]{Department of Mathematical Analysis, Institute of Mathematics,
Budapest University of Technology and Economics, 
M\H{u}egyetem rkp. 3., H-1111 Budapest, Hungary,
(jtoth@math.bme.hu)}
\date{}

\begin{document}
\maketitle

\begin{abstract}
A novel stochastic optimization method called MAC was suggested. The method is based on the calculation of the objective function at several random points and then an empirical expected value and an empirical covariance matrix are calculated. The empirical expected value is proven to converge to the optimum value of the problem. The MAC algorithm was encoded in Matlab and the code was tested on 20 test problems. Its performance was compared with those of the interior point method (Matlab name: fmincon), simplex, pattern search (PS), simulated annealing (SA), particle swarm optimization (PSO), and genetic algorithm (GA) methods. The MAC method failed two test functions and provided inaccurate results on four other test functions. However, it provided accurate results and required much less CPU time than the widely used optimization methods on the other 14 test functions.    
\end{abstract}

\section{Introduction} \label{sec:intro}
Optimization is one of the central topics of both applied and theoretical mathematics. It studies the problem of maximizing or minimizing outputs of functions related to the selected parameters. Using efficient optimization, one can in general increase efficiency and decrease losses. Many classes of optimization problems have been studied in mathematics, computer science, and other fields of science (e.g. see refs\ \cite{hopfield1985neural,krentel1986complexity,gambella2021optimization}).  Over the years, a plethora of specialized optimization areas emerged (e.g.\ convex optimization, stochastic programming, combinatorial optimization, etc.) and various methods have been elaborated and applied to specific problems. The algorithms can be classified into two groups based on whether they use randomization (stochastic methods) or not (deterministic methods).\cite{}\
Stochastic optimization methods (e.g. \cite{Fouskakis-2002,spall2012stochastic,Zheng-2015,Reddy-2017}) search for an optimal solution using the concept of probabilistic translation rules (randomness). These algorithms are gaining popularity due to certain properties which deterministic algorithms do not have. For example, stochastic algorithms may suit to use for objective functions that have multiple local optima whereby deterministic algorithms may get stuck.

An extremum (a maximum or minimum point) can be either global (the highest or lowest function value within a region) or local (the highest or lowest value in a neighborhood but not necessarily in the entire region). Finding a global extremum, in general, proves to be an onerous task while trying to explore a local extremum is usually less burdensome  
(see \cite{pressteukolskyvetterlingflannery},
\cite{zhigljavskyzilinskas}, \cite{weise} and \cite{hendrixgazdagtoth},). 
Several optimization methods have already been built into popular program packages like \Mma\ (Wolfram language) or Matlab.

The least-square function is a non-linear function of several variables (called parameters), which is to be optimized on a convex and compact parameter space. The least-square optimization can be regarded as one of the most studied areas that are still investigated and applied in many different areas of science and engineering (see \cite{bard,edalarra,seberwild}). The tasks of optimization are either maximization or minimization, but they are trivially related to each other, hence from now we are focusing on and formalizing only minimization problems. We are interested in minimizing a real-valued continuous function defined over a convex, compact domain.
The solutions to these types of problems are often based on finding the minimum of some kind of distance between the model results and the corresponding measurements \cite{bard,gustavoszekelytoth}.

In the applications of chemical kinetics to combustion systems, optimization of rate parameters of reaction mechanisms is required for the efficient interpretation of experimental data \cite{goitom2022efficient}. The rate parameters describe the temperature and pressure dependence of the rate coefficients of the reaction steps, and their initial values are assigned based on direct measurements of rate coefficients, theoretical calculations, or analogies to similar reaction steps. The mechanism is then tested against indirect experimental data, like measured laminar burning velocities and ignition delay times, through an optimization process. This optimization is carried out, by comparing the experimental and simulated data points using a nonlinear objective function called an error function. This error function have typically many local minima and the evaluation of the error function is computationally expensive because it is based on the solution of large systems of ODEs and PDEs. A search for efficient optimization methods applicable to chemical kinetics optimizations is one of the motivations of this paper.

As our focus is on the optimization of a function over a domain, the two main important issues to investigate are:
\begin{itemize}
\item[(i)] how many times we have to evaluate the error function and 
\item[(ii)] how many iterations are needed to arrive at or close to the extremum point?
\end{itemize}

In the present work, we describe a novel method that belongs to the family of stochastic approximation methods. The method generates a converging series of mean values and (empirical) covariance matrices, therefore it was named the MAC method. It seems to have good properties from the above-enumerated viewpoints. 

Convergence of the method will be proved in Theorem \ref{thm:convbounded}. 
We have chosen a series of benchmark functions found in  \cite{jamil2013literature,mishra2006performance,rahnamayan2007novel,shang2006note} and we have found that on several test problems, the novel method requires much fewer function evaluations, thus it is appropriate to solve problems when the function evaluation is expensive. Also, it approaches close to the minimum in fewer steps than several other methods.

The structure of the paper is as follows. Section \ref{subsec:algo} introduces the main concepts and presents the algorithm. After that (Section \ref{subsec:results}) the mathematical results are discussed. 
This is followed by the presentation of the numerical results on benchmark functions in Section \ref{sec:benchmark}.
We summarize the results and outline the plans for the near future in Section \ref{sec:summary}.

\section{Main concepts}\label{sec:main}

This section is devoted to presenting our stochastic optimization algorithm (MAC) in rigorous terms. After a preliminary discussion on the problem, we introduce the main algorithm and later we discuss some results.

\subsection{The algorithm}\label{subsec:algo}
We are given a \emp{continuous}, \emp{deterministic}, \emp{real-valued} function $Q$ of $d$ independent variables ($\xb=(x_1,x_2,\ldots,x_d)^\T, d\in\Zbb^{+}$) defined on a domain $\Dc\in\R{d}$. We assume that $\Dc$ is a \emp{convex}, \emp{compact} (i.e.\ closed and bounded) set of $\R{d}$. Under these terms, it follows that $Q$ is a bounded function and its (global and local) minimum(s) and maximum(s) are attained over $\Dc$. 

Our main goal is to find the extrema of $Q$ with a stochastic iterative algorithm. Minimizing $Q$ over $\Dc$ is equivalent to maximizing $-Q$ above the same region. Hence, without loss of generality, we are targeting to find the (global) minimum points of $Q$ on $\Dc$. Note that $Q$ is equivalently called \emp{objective function} throughout the article.

Under the above terms, let us also define the (non-empty) set $\Gc$ of \emp{global minimum} points of $Q$ over $\Dc$ as
\begin{equation}
\Gc := \left\{\xb\in\Dc:\, Q(\xb)\leq Q(\yb)\,\mbox{ holds for all $\yb\in\Dc$}\right\}.\label{eq:globmins}
\end{equation}

Let us introduce the \emp{mass} (or \emp{penalty}) \emp{function} $g$ that is associated with the \emp{objective function} $Q$ in such a way that for $(\rho,\xb)\in\R{+}_0\times \R{d}$:
\begin{equation}\label{eq:massfunc}
g(\rho,\xb) = \e{-\rho\, Q(\xb)}.
\end{equation}
Note that $0 < g(\rho,\xb)\leq 1$ holds for every $\rho\in\R{+}_0$ and $\xb\in\Dc$.

Next, we introduce the concept of \emp{ambient (apriori) distribution}. This is simply defined to be any fixed $\nub$ \emp{absolutely continuous}, $d$-dimensional probability distribution such that:
\begin{itemize}
    \item the support of $\nub$ is a convex, compact subset of $\R{d}$;
    \item $\nub$ is 'standard', that is its expectation is $\nullv$ and its covariance matrix is the identity matrix $\I$.
\end{itemize}
E.g.\ $\nub$ can be chosen to be the standard $d$-dimensional normal distribution (support of which is $\R{d}$) or it can be the uniform distribution over the $d$-dimensional unit 'ball' around the origin (that is its support is $\{\xb: \norm{\xb}_2\leq 1\}$). 

Finally, let $\aph, \gamma:\Zbb^{+}_0\times\Zbb^{+}_0\raw\Zbb^{+}_0$ be two \emp{non-decreasing, unbounded, double sequences} of natural numbers, that is for all $m \leq n,\, M \leq N\in\Zbb^{+}_0$: we have $\aph(m,M)\leq\aph(n,N)$, $\gamma(m,M)\leq\gamma(n,N)$, and $\sup_{l\in\Zbb^+}\aph(l,N)=\sup_{L\in\Zbb^+}\aph(n,L)=\sup_{l\in\Zbb^+}\gamma(l,N)=\sup_{L\in\Zbb^+}\gamma(n,L)=+\infty$ for all fixed $n, N\in\Zbb^{+}_0$. We also fix $\aph(0,0)$=0.\\

Now, we are ready to present the pseudo-code of our stochastic algorithm.
\begin{enumerate}
\item[] \emp{BEGIN}
\item let $\beta: \Zbb^{+}_0\raw\Zbb^{+}_0$ be a fixed, non-decreasing function of natural numbers such that $\beta(0)=0$ and let $\delta>0$ be a real number; with a slight abuse of notation, we use the shorthand notations below $\aph_n:=\aph(n,\beta(n))$ and $\gamma_n:=\gamma(n,\beta(n))$ for $n\in\Zbb^{+}_0$.
\item let $\ub_0\in\Dc$, pick a positive definite matrix $\Ub_0\in\R{d\times d}$, and fix an ambient (apriori) distribution $\nub$;
\item $n:=1$;
\begin{enumerate}[i.]
\item[] \emp{REPEAT}
\item 
Take a completely independent sample  $\left(\xi_i\right)_{i=\aph_{n-1}+1}^{\aph_n}$ from the distribution $\nub$ and let: 
\begin{equation*}
\pb_i:=\ub_{n-1}+\Ub_{n-1}\xi_i,
\end{equation*}
where $i=\aph_{n-1}+1,\aph_{n-1}+2,\ldots,\aph_n$. (Note that $\pb_{i}$'s are already calculated in the previous step(s) for $i=1,2,\ldots,\aph_{n-1}$.)
Finally, compute the weights:
$(\pb_i)_{i=\aph_{n-1}}^{\aph_n}$.
\item 
Compute the \emp{empirical expected value} and \emp{empirical covariance matrix} with those weights determined above, i.e.:
\begin{align}
\ub_n & := \sum_{i=1}^{\aph_n} \frac{g(\gamma_n,\pb_i)}
{\sum_{j=1}^{\aph_n} g(\gamma_n,\pb_j)}\, \pb_i, \label{eq:expiter}\\
\Ub_n &:= 
\left(
\sum_{i=1}^{\aph_n} \frac{g(\gamma_n,\pb_i)}
{\sum_{j=1}^{\aph_n} g(\gamma_n,\pb_j)}\, (\pb_i-\ub_n)(\pb_i-\ub_n)^\T 
\right)^{1/2},
\label{eq:variter}
\end{align}
where recall the penalty function defined in Eq. \ref{eq:massfunc}.
\item $n:=n+1$;
\item[] \emp{UNTIL} $\{\norm{\ub_{n}-\ub_{n-1}}_2 < \delta\}$.
\end{enumerate}
\item[] \emp{END}
\end{enumerate}
\begin{remark}
The $\aph_n$ and $\gamma_n$ parameters are also called the 'learning' parameters. The more we tweak them, i.e.\ increase, the closer we might get to the optimum (minimum) point.

Note that '$\Ub_n^2$' is a symmetric, positive semi-definite matrix, hence its square root is well-defined and is real as it can be decomposed as $\Ob\mathbf{\Lambda}\Ob^{\T}$ for some orthogonal matrix $\Ob\in\R{d\times d}$, where $\mathbf{\Lambda}$ is the diagonal matrix of the eigenvalues of $\Ub_n^2$, hence $\Ub_n=\Ob\mathbf{\Lambda}^{1/2}\Ob^{\T}$.

Mind that $\ub_n\in\Dc$ holds for all $n\in\Zbb^{+}$ due to the convexity of the set $\Dc$, hence, $\{\ub_n\}_{n\in\Zbb^+}$ is a bounded sequence in $\R{d}$. However, $\{\Ub_n\}_{n\in \Zbb^+}$ is only bounded when the support of $\nub$ is finite.
\end{remark}
Now, the immediate question of whether the above method converges at all, more precisely $\ub_n$ converges to an extremum point, is going to be answered affirmatively under some further conditions.

\subsection{Results}\label{subsec:results}

The main theorem of this section concerns the proof of convergence for the previously introduced algorithm under some further conditions on its parameters. 

\begin{thm}\label{thm:convbounded}
Let $\ub_0\in\Dc$ and $\Ub_0\in\R{d\times d}$ be a column-vector and an arbitrarily chosen positive definite matrix, respectively. 
Assume that the ambient (apriori) distribution $\nu$ is a.s. \emp{bounded} in such a way that 
for some $0<T<\frac{1}{4}$:
\begin{equation}\label{eq:finsupprt}
\nu\left(\Sc\right)=1,\,\mbox{ where $\,\Sc=\{\xb:\norm{\xb}_2^2\leq T\}$}.
\end{equation}
Let $\Ac:=\{\xb\in \Sc\,:\,\ub_0 + \Ub_0\xb\in \Dc\}$ and assume that $Q$ has a single minimum value over $\ub_0+\Ub_0 \Ac$, i.e.:
\begin{align*}
\argmin_{\xb\in \Ac}Q(\ub_0+\Ub_0\xb)&=\{\oxb\}.
\end{align*}
In other words, $\Gc=\{\ub_0+\Ub_0\oxb\}$, where recall \ref{eq:globmins}.
Furthermore, with a slight abuse of notation, let $\aph_n=N\cdot \tilde{\aph}_n$ and $\gamma_n=\gamma\cdot \tilde{\gamma}_n$, where $N\in\Zbb^+$, $\gamma\in\R{+}$ are free parameters, and $\tilde{\aph}, \tilde{\gamma}$ alone satisfy the requirements of the algorithm listed above. 
Then, we have the following convergence results
\begin{gather}
\lim_{(n,N,\gamma)\to+\infty}\ub_n=\ub_0+\Ub_0\oxb,\label{eq:expconv}\\
\lim_{(n,N,\gamma)\to+\infty}\Ub_n=\nullv\nullv^\T\label{eq:varconv}
\end{gather}
to hold such that for Eq. \ref{eq:expconv} we impose the condition $\lim_{(N,\gamma)\raw+\infty}N\cdot\e{-\gamma}=0$, where the above convergence is meant almost surely (in short: a.s.).
\end{thm}
\begin{proof}{(Theorem \ref{thm:convbounded}).}
First, we prove the second convergence \ref{eq:varconv} (for empirical variance) and then focus on the first one. First, let $c=\sqrt{4 T}$ and note that by our assumption \ref{eq:finsupprt}: $0<c<1$ holds.

It is straightforward to see from Eq. \ref{eq:variter} that
\begin{align*}
\Ub_n^2 = \Ub_{n-1}\left(\sum_{i=1}^{\aph_n}w_{i,\aph_n,\gamma_n}\wb_i\wb_i^\T\right)\Ub_{n-1},
\end{align*}
where we used the below short-hand notations:
\begin{align*}
\wb_i&=\xi_i - w_{i,\aph_n,\gamma_n}\xi_i,\\
w_{i,\aph_n,\gamma_n}&=\frac{g(\gamma_n,\pb_i)}{\sum_{j=1}^{\aph_n} g(\gamma_n,\pb_j)} 
\end{align*}
Recall that the random variables $\xi_i$'s are directly sampled from $\nu$, hence $\xi_i \in \Sc$ holds for every $1\leq i\leq\aph_n$.
Then, by \ref{eq:finsupprt}, the following chain of inequality can easily be derived:
\begin{align*}
\norm{\wb_i\wb_i^\T}_F
&\leq
\sup_{\substack{\kappa_i\in[0,1], \xb_i\in \Sc \\ \sum_{i=1}^{\aph_n}\kappa_i=1, \norm{\xb_i}_2^2\leq T \\ (1\leq i\leq\aph_n)}}
\norm{\left(\xb_i-\sum_{j=1}^{\aph_n}\kappa_i\xb_i\right)\left(\xb_i^\T-\sum_{i=1}^{\aph_n}\kappa_i\xb_i^\T\right)}_F\\
&\leq \sup_{\xb\in\Sc}\norm{2\xb}_2^2 = 4T = c^2,
\end{align*}
holding for all $1\leq i\leq \aph_n$, where $\norm{\cdot}_F$ denoted the Frobenius norm.
Now, from the above estimate it is clear that
\begin{align}
\norm{\Ub_{n}}_F^2
&=\Tr\Ub_n^2\nonumber\\
&=\Tr\Ub_{n-1}\left(\sum_{i=1}^{\aph_n}w_{i,\aph_n,\gamma_n}\wb_i\wb_i^\T\right)\Ub_{n-1}\nonumber\\
&=\sum_{i=1}^{\aph_n}w_{i,\aph_n,\gamma_n}\Tr\left(\wb_i\wb_i^\T\right)\Ub_{n-1}^2\nonumber\\
&\leq\sum_{i=1}^{\aph_n}w_{i,\aph_n,\gamma_n}\abs{\SPr{\wb_i\wb_i^\T}{\Ub_{n-1}^2}_F}\nonumber\\
&\leq\sum_{i=1}^{\aph_n}w_{i,\aph_n,\gamma_n}\norm{\wb_i\wb_i^\T}_F\norm{\Ub_{n-1}^2}_F\nonumber\\
&=\norm{\Ub_{n-1}}_F^2 \sum_{i=1}^{\aph_n} w_{i,\aph_n,\gamma_n} \norm{\wb_i}_2^2\nonumber\\
&\leq c^2 \norm{\Ub_{n-1}}_F^2 \sum_{i=1}^{\aph_n} w_{i,\aph_n,\gamma_n} (1 - w_{i,\aph_n,\gamma_n})\nonumber\\
&=c^2 \norm{\Ub_{n-1}}_F^2 \left(1 - \sum_{i=1}^{\aph_n} w_{i,\aph_n,\gamma_n}^2 \right)\label{eq:varestim}
\end{align}
holds for every $n\in\Zbb^+$, where we used the cyclicity and linearity of the trace, the Cauchy--Bunyakovsky--Schwarz inequality, and the sub-multiplicative property of the Frobenius norm along the way. 

Now, using the inequality just obtained in \ref{eq:varestim}, an induction on $n$ leads to the following estimates for $\Ub_n$:
\begin{align}
\norm{\Ub_{n}}_2 &\leq \norm{\Ub_{n}}_{F}\nonumber\\
&\leq c^{n-1} \norm{\Ub_1}_F\nonumber\\
&\leq c^{n}\norm{\Ub_{0}}_{F}\sqrt{1 - \sum_{i=1}^{\aph_n}w_{i,\aph_n,\gamma_n}^2}\nonumber\\
&\leq c^n \norm{\Ub_{0}}_{F} \label{eq:varestim1}
\end{align}
that hold for every $n\in\Zbb^+$, taking advantage of the subordinate property of the Frobenius norm to the 2-norm and the fact that $\nu$ is being finitely supported here.
This then implies that $\Ub_n\to\nullv\nullv^\T$ as $(n,N,\gamma)\to+\infty$ a.s. Since each norm is equivalent in finite dimension, $\Ub_n$ converges to $\nullv\nullv^\T$ a.s. in any (matrix) norm completing the proof of Eq. \ref{eq:varconv}.

Now, turning to the proof of Eq. \ref{eq:expconv}, observe that using the above notation we have 
\begin{align}
\ub_{n} = \ub_{n-1} + \Ub_{n-1}\sum_{i=1}^{\aph_n} w_{i,\aph_n,\gamma_n} \xi_i\label{eq:expiter2}
\end{align}
to hold for every $n\in\Zbb^+$. As a consequence
\begin{align*}
\norm{\ub_{n}-\ub_{n-1}}_2&\leq\norm{\Ub_{n-1}}_2 \norm{\sum_{i=1}^{\aph_n} w_{i,\aph_n,\gamma_n} \xi_i}_2\\
&\leq \norm{\Ub_{n-1}}_2 \sup_{\xb\in\Sc} \norm{\xb}_2 \leq c \norm{\Ub_{n-1}}_F \leq c^{n}\norm{\Ub_0}_F,
\end{align*}
using the previously deduced inequality for $\Ub_n$ (Ineq. \ref{eq:varestim1}), and the relation between the Frobenius- and the 2-norm of a matrix.
With the above inequality under the belt, it is not hard to see that $\ub_n$ is indeed a Cauchy sequence. 
Let $n,m\in \Zbb^+$ such that $n>m$, then it holds that
\begin{align}
\norm{\ub_n-\ub_m}_{2}
&\leq\sum_{i=1}^{n-m}\norm{\ub_{m+i}-\ub_{m+i-1}}_{2}\nonumber\\
&\leq \norm{\Ub_0}_F\sum_{i=1}^{n-m}c^{m+i}\leq \norm{\Ub_0}_F\frac{c^m}{1-c},\label{eq:cauchy}
\end{align}
which goes to zero as $(n,m)\raw+\infty$. It then follows that there exists a $\ub_{\infty}$ random variable such that $\ub_n\raw\ub_{\infty}$ as $n\raw+\infty$ a.s., and
\begin{align}
\norm{\ub_n-\ub_{\infty}}_{2}\leq C\cdot c^{n},\label{eq:convestim}
\end{align}
where $C=\frac{1}{1-c}\norm{\Ub_0}_F$ is an absolute constant depending only on the choice of $\Ub_0$, that is no matter how we have chosen $(N,\gamma)$.

Note that Eq. \ref{eq:expiter2} can be further written as:
\begin{align}
\ub_n = \ub_0 + \Ub_{0}\sum_{i=1}^{\aph_1}w_{i,\aph_1,\gamma_1} \xi_i + \sum_{j=2}^{n}\Ub_{j-1}\sum_{i=1}^{\aph_j}w_{i,\aph_j,\gamma_j} \xi_i\quad (n > 0),\label{eq:expiterexpand}
\end{align}
and from this it is easy to see that for each fixed $n\in\Zbb^+$, the sequence $(\ub_n)_{(N,\gamma)}$ as $(N,\gamma)\raw+\infty$ converges a.s. using the strong law of large numbers. Hence, the limit $\lim_{(n,N,\gamma)\raw+\infty}\ub_n$ exists a.s., taking advantage of the uniform estimate pf Ineq. \ref{eq:convestim}. Finally, using the fact that $\sum_{i=1}^{\aph_1}w_{i,\aph_1,\gamma_1}^2$ goes to $1$ as $(N,\gamma)\raw+\infty$ such that we have $N\cdot\e{-\gamma}\raw 0$ in place by our assumption, and by recalling Ineq. \ref{eq:varestim1}, we can conclude that the double sum (third term) of Eq. \ref{eq:expiterexpand} converges to zero, whereas the rest approaches the global minimum.
\end{proof}

\section{Benchmarking}\label{sec:benchmark}

The new method was tested on 18 artificial optimization test functions and the results were compared with results obtained from a series of widely used numerical optimization methods.

\subsection{Setup for benchmarking}\label{subsec:setup}
The newly proposed stochastic global optimization method, to be referred to as the MAC method was implemented in the Matlab programming language. MAC refers to empirical Mean And empirical Covariance matrix as it was based on a series of calculations of empirical mean and empirical covariance matrix.  
Several nonlinear uni-modal and multi-modal benchmark functions were selected to compare the performance of the method with other well-known numerical optimization methods.
The functions given by an explicit formula that is used for this benchmarking are outlined in Table \ref{tab:funs}, below. 
Recently, \cite{layeb} has found a few functions that are even harder to minimize that are included in Table \ref{tab:funs} (f9 to f18).

\begin{table}[h!]
    \centering
    \begin{tabular}{llcl}
         &  \\
\text{ID} & \text{Function Name} & \text{Dim} & \text{Domain}\\
    \hline
    \hline
f1 & Ackley & 10 & $\left[-32.768\ 32.768\right]^{5}$\\
f2 & Cross-in-tray & 2 & $\left[-10\ 10\right]^{2}$\\
f3 & Rastrigin & 10 & $\left[-5.12\ 5.12\right]^{5}$\\
f4 & Rosenbrock & 10 & $\left[-5\ 10\right]^{5}$\\
f5 & RosenbrockSmall & 10 & $\left[-2.048\ 2.048\right]^{5}$\\
f6 & Rosenbrock scaled & 4 & $\left[0\ 1\right]^{4}$\\
f7 & Sphere & 20 & $\left[-5.12\ 5.12\right]^{5}$\\
f8 & Zakharov & 30 & $\left[-5\ 10\right]^{5}$\\
f9 & Layeb01 & 5 & $\left[-100\ 100\right]^{5}$\\
f10 & Layeb02 & 5 & $\left[-10\ 10\right]^{5}$\\
f11 & Layeb03 & 5 & $\left[-10\ 10\right]^{5}$\\
f12 & Layeb04 & 5 & $\left[-10\ 10\right]^{5}$\\
f13 & Layeb10 & 5 & $\left[-100\ 100\right]^{5}$\\
f14 & Layeb11 & 5 & $\left[-10\ 10\right]^{5}$\\
f15 & Layeb12 & 5 & $\left[-100\ 100\right]^{5}$\\
f16 & Layeb17 & 5 & $\left[-100\ 100\right]^{5}$\\
f17 & Layeb19 & 5 & $\left[-5\ 5\right]^{5}$\\
f18 & Layeb20 & 5 & $\left[-5\ 5\right]^{5}$\\
    \end{tabular}
        \caption{Function ID, dimension, and domain of test functions}
        \label{tab:funs}
\end{table}
Details of the test functions (Table 1) are found in the following references \cite{jamil2013literature,mishra2006performance,rahnamayan2007novel,shang2006note} and \cite{layeb}.
The MAC method is an iterative technique that starts from some initial values of its hyper-parameters. These initial values of the hyper-parameters depend on the user and problem type. There are two important parameters that control the whole method (hence it has an effect on its convergence, as well), they are:
\begin{enumerate}
\item 
the sample size $\aph_n$ that is generated at every step $n$, and
\item
another 'learning' parameter denoted by $\gamma_n$.
\end{enumerate}
\newpage
As discussed in the pseudo-code, the MAC method has several tuning parameters. In this implementation and testing of the MAC method, the following initial values of these tuning parameters and their updating rules were used.
\begin{align*}
& N = 4,\quad n = 0\\
&\gamma_{0} := 0.001,\quad
\gamma_{n} := 2.8 \times \gamma_{n-1}\\ 
&\aph_0 := 1,\quad\aph_n := \aph_{n-1} + n \times N. 
\end{align*}

\subsection{Benchmark results}\label{subsec:benchmark}

The final calculated results of the different numerical methods and the MAC method for the benchmark functions are given as tables in the following pages. The interior point (Matlab name fmincon) and the simplex methods were used from the MATLAB Optimization Toolbox, while the pattern search (PS), simulated annealing (SA), particle swarm optimization (PSO), and genetic algorithm (GA) methods were used from the MATLAB Global Optimization Toolbox. All the tested methods were used with their default method parameters. 
In Table \ref{tab:best}, the minimum values are displayed, in Table \ref{tab:iter}, the number of function evaluations required to obtain the minimum value, and in Table \ref{tab:timing}, the elapsed time (sec) required are displayed. 
The calculations were performed on a Windows 10 PC, Intel(R) Core(TM) i5-8250U CPU @ 1.80 GHz, RAM 8.00 GB. 
\eject
\begin{landscape}
\begin{table}[h!]
    \centering
    \begin{tabular}{|l|c|c|c|c|c|c|c|c|} 
         \hline
\text{Function} & \text{fmincon} & \text{simplex} & \text{PS}  & \text{SA} & \text{PSO} & \text{GA} & \text{MAC} & \text{Exact}\\
    \hline
Ackley & $\mathrm{18.2}$ & 9.0 & $\mathrm{1.1E-4}$ & 15.6 & $\mathrm{2.2E-7}$ & $\mathrm{2.2E-5}$ & $\mathrm{8.0E-3}$ & 0\\
    \hline
Cross-in-tray & -1.7906 & -1.7906 & -2.0626 & -2.0626 & -2.0626 & -2.0626 & -2.0626 & -2.0626\\
    \hline
Rastrigin & fail & 74.6 & $\mathrm{5.9E-9}$ & 36.8 & 8.0 & $\mathrm{2.9E-7}$ & $\mathrm{2.7E-4}$ & 0\\
    \hline
Rosenbrock & $\mathrm{1.7E-11}$ & $\mathrm{6.9E-9}$ & $\mathrm{2.6E-1}$ & $\mathrm{3.0E-1}$ & $\mathrm{8.3E-4}$ & $\mathrm{9.5E-2}$ & $\mathrm{8.5E-1}$ & 0\\
    \hline
RosenbrockSmall & $\mathrm{6.4E-11}$ & $\mathrm{7.8E-24}$ & $\mathrm{1.7E-4}$ & 6.7 & 4.0 & $\mathrm{8.1E-4}$ & $\mathrm{1.6E-2}$ & 0\\
    \hline
RosenbrockScaled & -1.0192 & -1.0192 & -1.0192 & -1.0191 & -1.0192 & -1.0191 & -1.0191\\
    \hline
Sphere & $\mathrm{1.1E-23}$ & $\mathrm{2.5E-16}$ & $\mathrm{3.0E-11}$ & $\mathrm{4.3}$ & $\mathrm{1.3E-7}$ & $\mathrm{1.1E-5}$ & $\mathrm{2.3E-4}$ & 0\\
    \hline
Zakharov & $\mathrm{5.4E-14}$ & fail & 74.6 & 9.5 & fail & $\mathrm{1.8E-4}$ & $\mathrm{3.7E-6}$ & 0\\
    \hline
Layeb01 & 3.4 & fail & 0 & $\mathrm{3.8E-3}$ & fail & $\mathrm{6.0E-2}$ & 51.4 & 0\\
    \hline
Layeb02 & $\mathrm{7.2E-1}$ & 0 & 0 & $\mathrm{6.0E-1}$ & $\mathrm{7.2E-1}$ & $\mathrm{5.8E-1}$ & 0 & 0\\
    \hline
Layeb03 & $\mathrm{-8.7E-4}$ & $\mathrm{-2.3E-3}$ & $\mathrm{-1.3E-3}$ & $\mathrm{-3.8E-4}$ & fail & $\mathrm{-3.6E-3}$ & $\mathrm{-6.1E-4}$ & -4\\
    \hline
Layeb04 & -31.631 & -11.561 & -31.631 & -7.761 & -31.631 & -31.631 & -23.6283 & -31.631\\
    \hline
Layeb10 & fail & $\mathrm{1.3E-24}$ & 61.5 & 61.5 & fail & 75.3 & $\mathrm{2.9E-1}$ & 0\\
    \hline
Layeb11 & fail & -2 & -4 & $\mathrm{-9.9E-1}$ & -2 & -3.0 & -3.6 & -4\\
    \hline
Layeb12 & -12.3545 & -13.7137 & -14.8731 & -13.0683 & -13.2480 & -14.7952 & -14.785 & -14.873\\
    \hline
Layeb17 & 4.0 & 16.8 & 0 & 57.6 & 4.0 & 4.0 & 4.3 & 0\\
    \hline
Layeb19 & $\mathrm{2.5E-16}$ & $\mathrm{3.4E-30}$ & 0 & fail & $\mathrm{1.9E-13}$ & $\mathrm{2.2E-11}$ & $\mathrm{2.6E-6}$ & 0\\
    \hline
Layeb20 & $\mathrm{1.4E-16}$ & $\mathrm{9.3E-28}$ & 0 & 27.7 & $\mathrm{2.0E-9}$ & $\mathrm{1.1E-6}$ & $\mathrm{2.8E-4}$ & 0\\
    \hline
    \end{tabular}
        \caption{The exact value versus the best value found by the different optimization methods.}
        \label{tab:best}
\end{table}
\end{landscape}
\eject

\begin{table}[!h]
    \centering
    \begin{tabular}{|l|c|c|c|c|c|c|c|} 
         \hline
\text{Function} & \text{fmincon} & \text{simplex} & \text{PS}  & \text{SA} & \text{PSO} & \text{GA} & \text{MAC}\\
    \hline
Ackley & 163 & 1877 & 3096 & 5621 & 11100 & 5787 & 1723\\
    \hline
Cross-in-tray & 24 & 191 & 234 & 1033 & 850 & 429 & 4810\\
    \hline
Rastrigin & fail & 840 & 3406 & 4681 & 8700 & 11803 & 4615\\
    \hline
Rosenbrock & 702 & 5652 & 10000 & 10001 & 230005 & 250563 & 4825\\
    \hline
RosenbrockSmall & 755 & 4851 & 20000 & 8311 & 230050 & 470053 & 773\\
    \hline
Rosenbrock Scaled & 400 & 567 & 1906 & 4511 & 2450 & 2920 & 2796\\ 
    \hline
Sphere & 339 & 16323 & 11701 & 13500 & 10550 & 16362 & 5431\\
    \hline
Zakharov & 2000 & fail & 20000 & 13441 & fail & 98988 & 2043\\
    \hline
Layeb01 & 541 & fail & 321 & fail & fail & 32436 & 4831\\
    \hline
Layeb02 & 78 & 60 & 76 & 8906 & 1300 & 9641 & 4810\\
    \hline
Layeb03 & 252 & 504 & 722 & 3676 & fail & 9547 & 3233\\
    \hline
Layeb04 & 644 & 555 & 704 & 3321 & 5300 & 6445 & 5431\\
    \hline
Layeb10 & fail & 22856 & 191 & 2501 & fail & 7949 & 4810\\
    \hline
Layeb11 & fail & 1880 & 153 & 5871 & 9050 & 5364 & 3233\\
    \hline
Layeb12 & 200 & 979 & 266 & 3616 & 10600 & 9359 & 3713\\
    \hline
Layeb17 & 480 & 2212 & 389 & 10306 & 3250 & 5834 & 4828\\
    \hline
Layeb19 & 5005 & 1079 & 121 & fail & 2800 & 3578 & 2400\\
    \hline
Layeb20 & 132 & 898 & 121 & 2626 & 3650 & 3390 & 4810\\
    \hline
    \end{tabular}
        \caption{The number of function evaluations until convergence.}
        \label{tab:iter}
\end{table}
\begin{landscape}
\begin{table}[h!]
    \centering
    \begin{tabular}{|l|c|c|c|c|c|c|c|} 
         \hline
\text{Function} & \text{fmincon} & \text{simplex} & \text{PS}  & \text{SA} & \text{PSO} & \text{GA} & \text{MAC}\\
    \hline
Ackley & $\mathrm{7.4E-1}$ & $\mathrm{1.8E-1}$ & $\mathrm{2.4E-1}$ & $\mathrm{2.1}$ & $\mathrm{4.7E-1}$ & $\mathrm{1.6}$ & $\mathrm{2.7E-1}$\\
    \hline
Cross-in-tray & $\mathrm{3.5E-2}$ & $\mathrm{7.1E-3}$ & $\mathrm{4.9E-2}$ & $\mathrm{3.3E-1}$ & $\mathrm{8.3E-2}$ & $\mathrm{3.1E-1}$ & $\mathrm{4.8E-1}$\\
    \hline
Rastrigin & fail & $\mathrm{1.3E-1}$ & $\mathrm{5.2E-1}$ & $\mathrm{2.4}$ & $\mathrm{3.0E-1}$ & $\mathrm{1.0}$ & $\mathrm{3.0}$\\
    \hline
Rosenbrock & $\mathrm{8.1E-1}$ & $\mathrm{3.4E-1}$ & $\mathrm{8.8E-1}$ & $\mathrm{11.2}$ & $\mathrm{3.3}$ & $\mathrm{13.6}$ & $\mathrm{9.2E-1}$\\
    \hline
RosenbrockSmall & $\mathrm{8.0E-1}$ & $\mathrm{5.4E-1}$ & $\mathrm{2.2}$ & $\mathrm{4.1}$ & $\mathrm{4.1}$ & $\mathrm{16.4}$ & $\mathrm{1.2}$\\
    \hline
RosenbrockScaled & $\mathrm{6.4E-1}$ & $\mathrm{7.5E-2}$ & $\mathrm{4.6}$ & $\mathrm{1.7}$ & $\mathrm{4.0E-1}$ & $\mathrm{7.1E-1}$ & $\mathrm{7.4E-1}$\\
    \hline
Sphere & $\mathrm{5.6E-1}$ & $\mathrm{6.9E-1}$ & $\mathrm{6.0E-1}$ & $\mathrm{6.5}$ & $\mathrm{3.0E-1}$ & $\mathrm{1.4}$ & $\mathrm{1.5}$\\
    \hline
Zakharov & $\mathrm{6.7E-1}$ & fail & $\mathrm{8.4E-1}$ & 8.0 & fail & 6.1 & $\mathrm{9.7E-2}$\\
    \hline
Layeb01 & $\mathrm{7.0E-1}$ & fail & $\mathrm{2.6E-1}$ & fail & fail & $\mathrm{1.9}$ & $\mathrm{1.7E-1}$\\
    \hline
Layeb02 & $\mathrm{4.5E-1}$ & $\mathrm{5.2E-2}$ & $\mathrm{3.3E-1}$ & 2.5 & $\mathrm{1.7E-1}$ & 1.1 & $\mathrm{3.9E-1}$\\
    \hline
Layeb03 & $\mathrm{5.3E-1}$ & $\mathrm{4.2E-1}$ & $\mathrm{6.9E-1}$ & 1.2 & fail & $\mathrm{8.8E-1}$ & $\mathrm{2.3E-1}$\\
    \hline
Layeb04 & $\mathrm{9.7E-1}$ & $\mathrm{6.2E-2}$ & $\mathrm{2.9E-1}$ & $\mathrm{9.9E-1}$ & $\mathrm{2.3E-1}$ & $\mathrm{6.9E-1}$ & $\mathrm{4.8E-1}$\\
    \hline
Layeb10 & fail & 1.1 & $\mathrm{9.1E-1}$ & $\mathrm{4.0E-1}$ & fail & $\mathrm{4.1E-1}$ & $\mathrm{8.8E-2}$\\
    \hline
Layeb11 & fail & $\mathrm{1.8E-1}$ & $\mathrm{6.8E-2}$ & 2.1 & $\mathrm{3.3E-1}$ & $\mathrm{7.6E-1}$ & $\mathrm{6.3E-2}$\\
    \hline
Layeb12 & $\mathrm{6.3E-1}$ & $\mathrm{8.4E-2}$ & $\mathrm{2.6E-1}$ & 1.6 & $\mathrm{3.4E-1}$ & 1.0 & $\mathrm{1.3E-1}$\\
    \hline
Layeb17 & $\mathrm{6.8E-1}$ & $\mathrm{1.4E-1}$ & $\mathrm{3.4E-1}$ & 2.9 & $\mathrm{7.4E-1}$ & $\mathrm{1.4E-1}$ & $\mathrm{1.2E-1}$\\
    \hline
Layeb19 & 2.7 & $\mathrm{2.5E-1}$ & 1.5 & fail & 2.2 & 1.3 & $\mathrm{6.8E-1}$\\
    \hline
Layeb20 & $\mathrm{6.0E-1}$ & $\mathrm{2.3E-1}$ & $\mathrm{3.3E-1}$ & $\mathrm{6.7E-1}$ & $\mathrm{6.9E-1}$ & $\mathrm{9.8E-1}$ & $\mathrm{9.4E-1}$\\
    \hline
    \end{tabular}
        \caption{Running time (sec) until convergence.}
        \label{tab:timing}
\end{table}
\end{landscape}

\begin{figure}[h!]
    \centering
    \begin{minipage}{0.49\textwidth}
    \centering
    \includegraphics[width =1.25\textwidth]{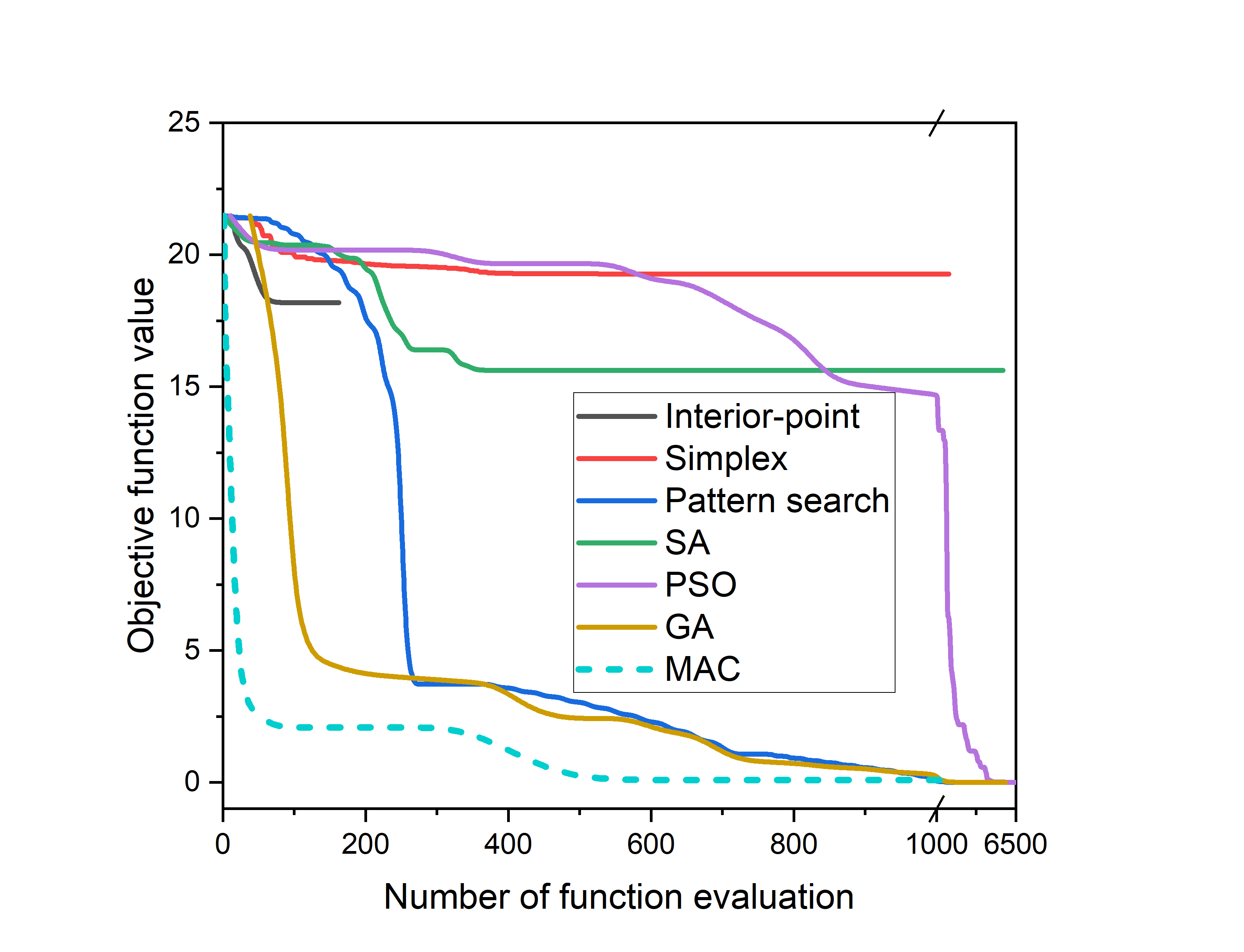}
    \caption*{Ackley 10D function}
    \label{}
    \end{minipage}
    \hfill
    \begin{minipage}{0.49\textwidth}
    \centering
    \includegraphics[width =1.25\linewidth]{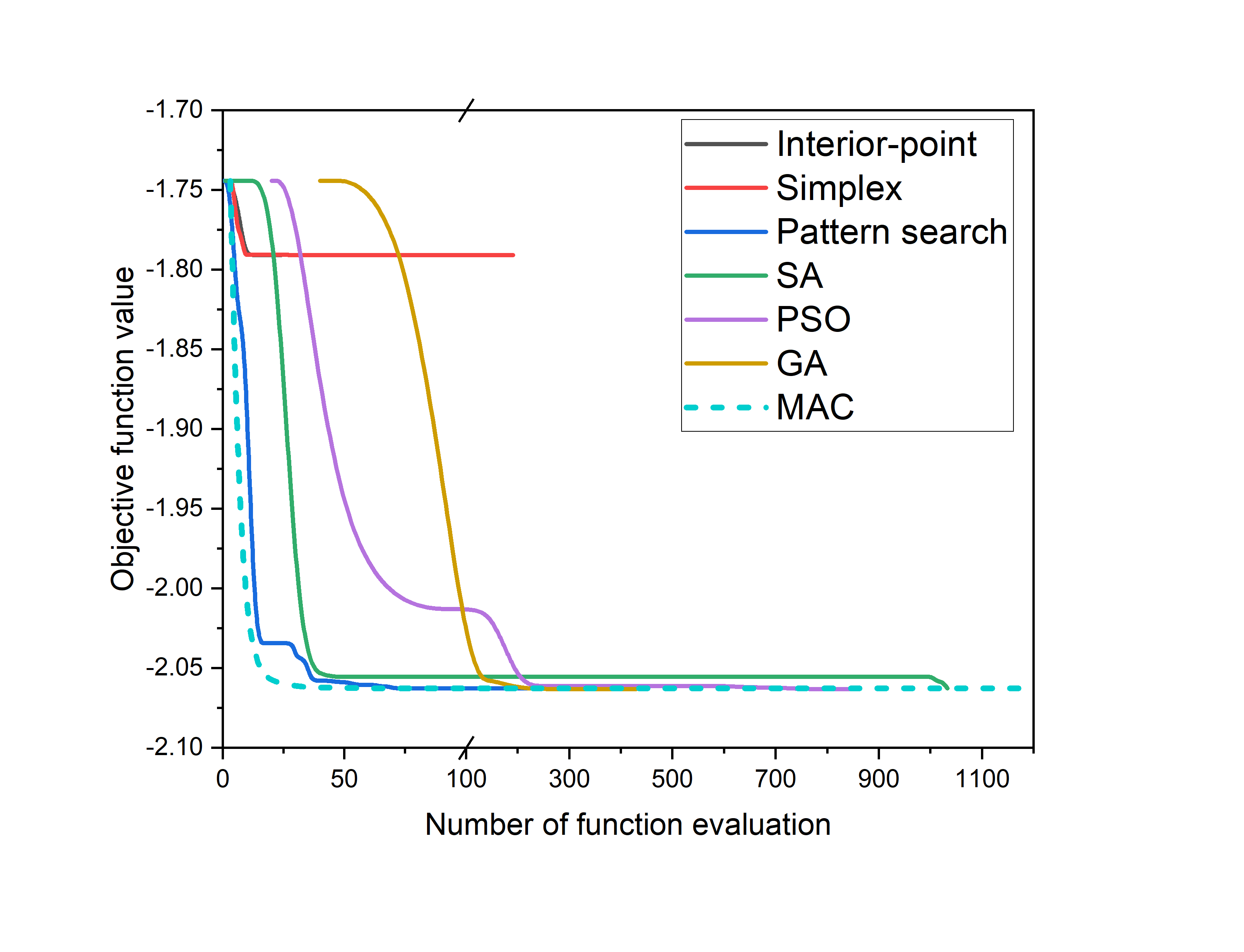}
    \caption*{Cross-in-tray 2D function}
    \label{}
    \end{minipage}
    \label{}
\end{figure}
\begin{figure}[h!]
    \centering
    \begin{minipage}{0.49\textwidth}
    \centering
    \includegraphics[width =1.20\textwidth]{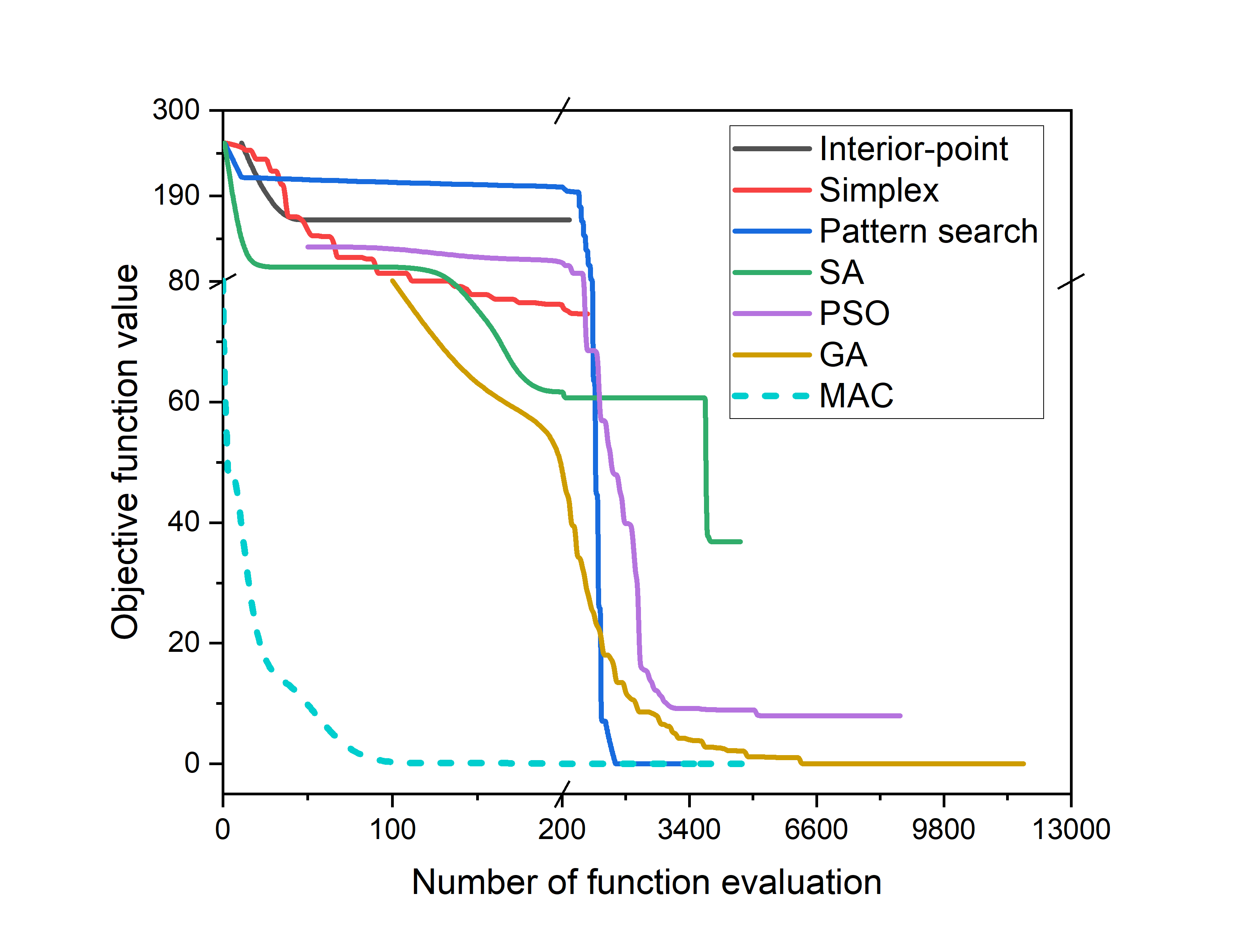}
    \caption*{Rastrigin 10D function}
    \label{}
    \end{minipage}
    \hfill
    \begin{minipage}{0.49\textwidth}
    \centering
    \includegraphics[width =1.25\linewidth]{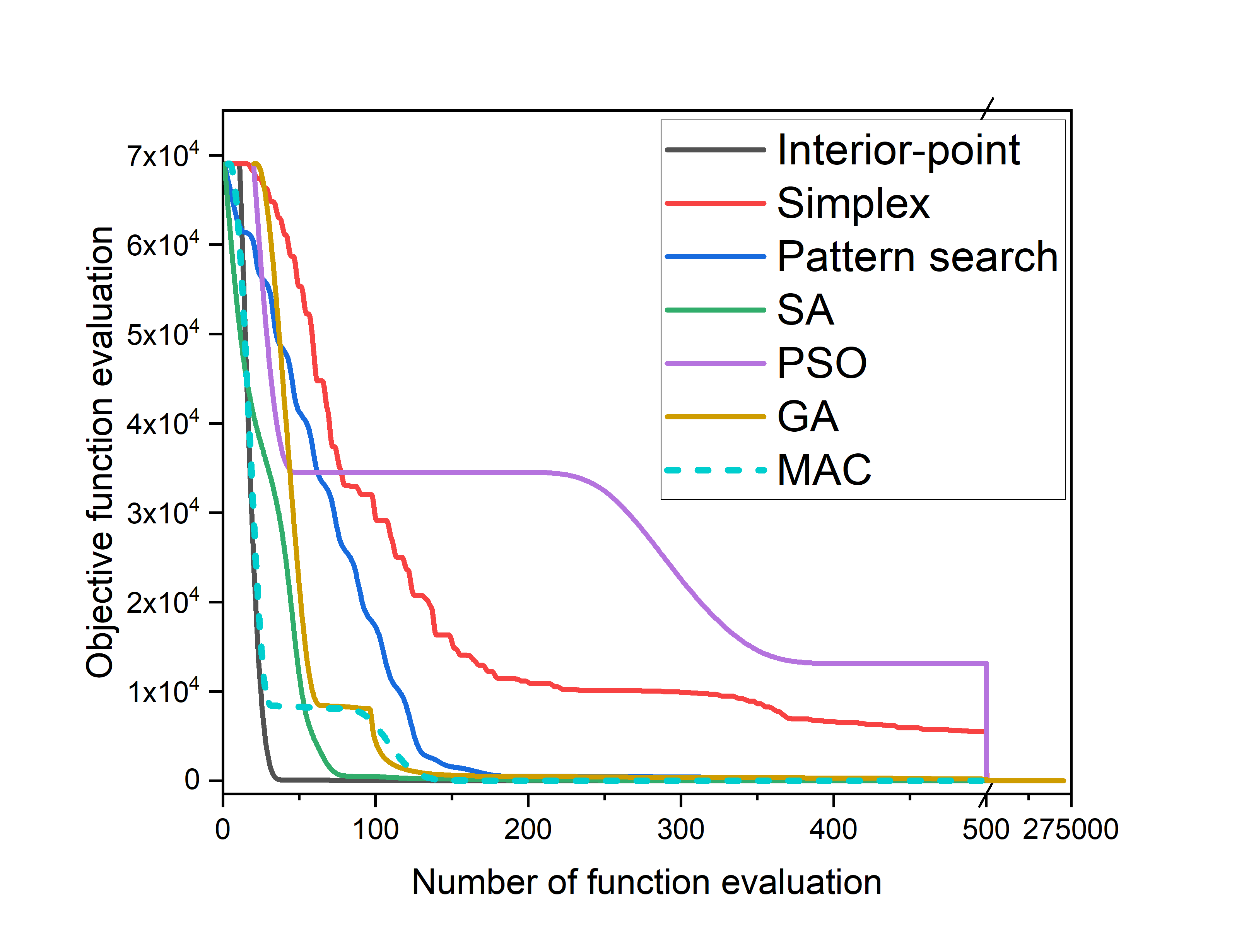}
    \caption*{Rosenbrock 10D function}
    \label{}
    \end{minipage}
\end{figure}
\begin{figure}[h]
    \centering
    \begin{minipage}{0.49\textwidth}
    \centering
    \includegraphics[width =1.15\textwidth]{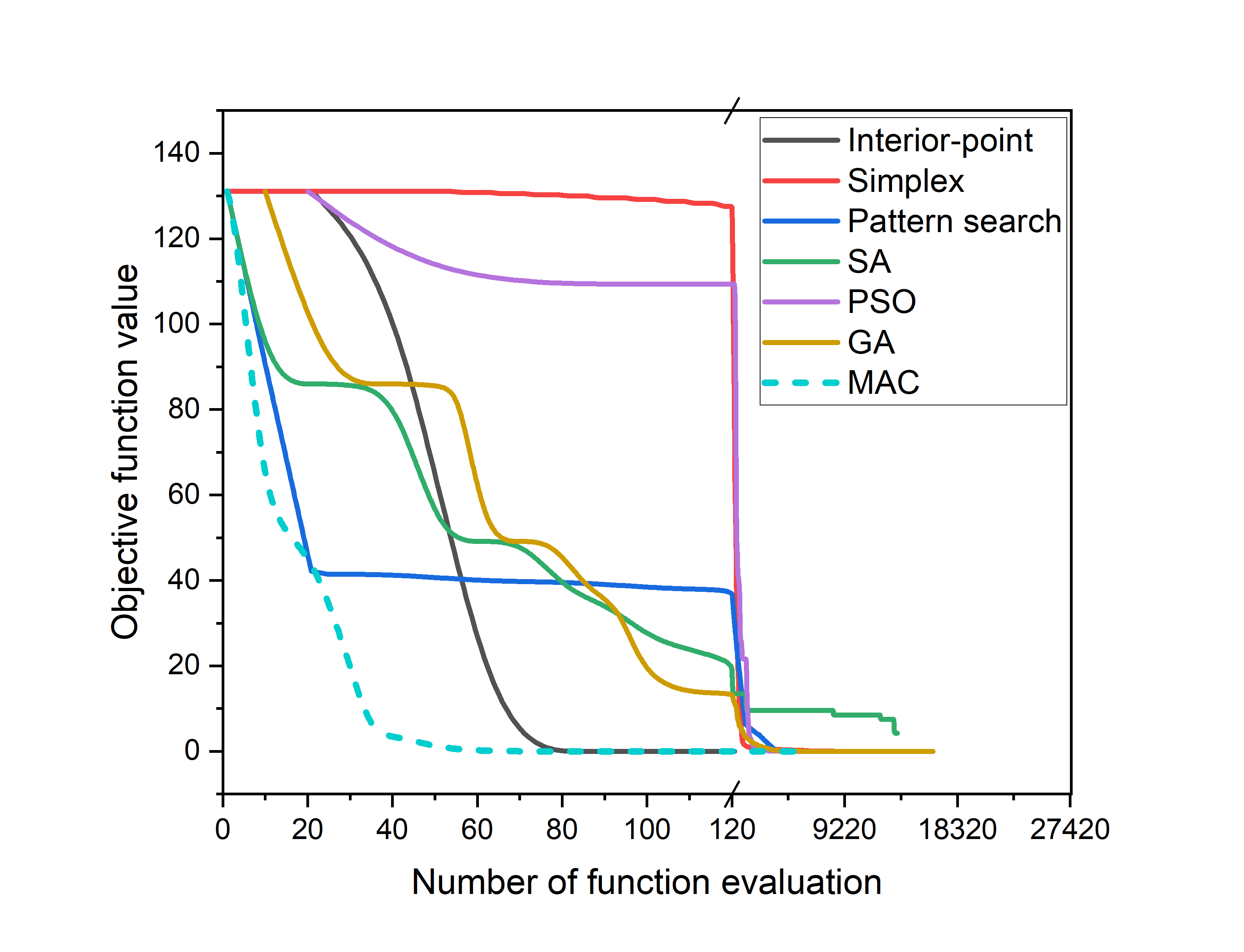}
    \caption*{Sphere 20D function }
    \label{}
    \end{minipage}
    \hfill
    \begin{minipage}{0.49\textwidth}
    \centering
    \includegraphics[width =1.25\linewidth]{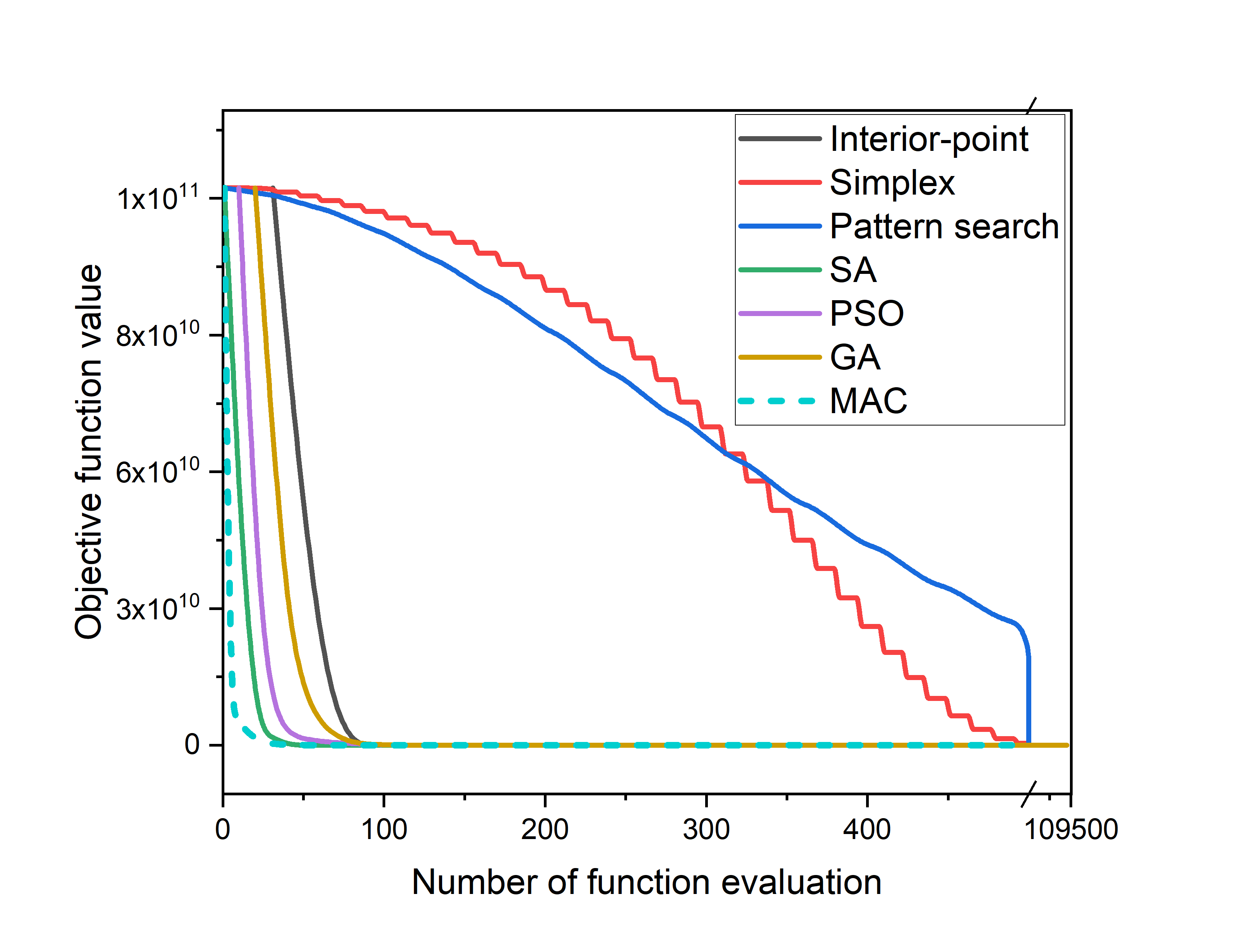}
    \caption*{Zakharov 30D function}
    \label{}
    \end{minipage}
    \caption{The objective function value corresponding to the number of function evaluation for different benchmark functions using MAC and other numerical optimization methods}
    \label{fig}
\end{figure}
\newpage

From the above tabular results, we can say that, the MAC method in general has good performance in most of the benchmark functions except in some of the most complex test functions. In some of the most complex test functions (f9, f11 and f16), the MAC method failed to find an optimum value of the objective function as all the other numerical methods were failed too. According to the summaries of the above tables (final objective function value), the MAC method outperformed almost all the other numerical optimization methods on the test functions: Rastrigin (f3), Zakharov (f8), Layeb02 (f10), Layeb10 (f13) and Layeb11 (f14). In addition to that, the MAC method outperformed the SA in all the test functions except in the Layeb01 (f9) test function. In terms of the number of function evaluations, in general, the MAC method required more function evaluation than all the local optimization methods and less function evaluation than the global optimization methods. The running time requirement in general does not have a significant difference among all the methods and all the test functions. However, in most of the test functions, the MAC method ended up with less running time (sec) than the other global optimization methods. 
\section{Summary and outlook}\label{sec:summary}
A ''good'' algorithm (or heuristics) must be very ''efficient'' in such a way as to be applicable to real-life problems, such as e.g. model fitting, maximum likelihood estimate, cost minimization, and performance maximization. 
What do we mean by ''efficient''? 
On the one hand, we have to take into consideration that $Q$ is probably expensive to evaluate, and we may or may not be able to efficiently and/or accurately compute its gradient or the Hessian matrix. 
So our task is rather to find a minimum point of $Q$ with as few evaluations as possible. 
On the other hand, we must recognize what we are looking for, we must have well-defined stopping criteria, convergence order, and a clear image of the code. 

New and new optimization methods are being published even nowadays. One reason for it is that numerical optimization methods are widely used for solving a series of optimization problems having arisen in science and technology. The development of the MAC stochastic optimization method presented in this paper was motivated by the task of the estimation of reaction rate parameters of large chemical kinetics models. Such models consist of large systems of ordinary or partial differential equations. The specialty of these optimization problems is that the prior domain of uncertainty of the fitted parameters is large and within this domain, the objective function contains a large number of local minima. Also, the evaluation of the objective function is very computationally expensive, since it includes the solution of many large systems of differential equations. Such optimization tasks arise in many fields including chemistry science and engineering, combustion science, and systems biology (e.g. development of metabolic models).
Our immediate plan is to test the MAC method on a real-life chemical kinetics problem related to combustion models.

\section{Supplementary material}

We are attaching as Supplementary material
\begin{enumerate}
\item The Matlab code of the MAC method
\item The benchmark functions. 
\end{enumerate}

\bibliography{_estimation}
\bibliographystyle{apalike}
\end{document}